\documentclass[12pt]{article}
\usepackage{natbib}
\usepackage{epsfig}
\usepackage{times}
\usepackage{amsmath}
\usepackage{amsthm}
\usepackage{amsfonts}
\usepackage{amssymb}
\usepackage{mathrsfs}
\usepackage{color}
\usepackage{algorithm}
\usepackage{algorithmic}

\DeclareMathOperator*{\argmax}{\arg\!\max} 
\DeclareMathOperator*{\argmin}{\arg\!\min} 
\DeclareMathOperator{\prox}{prox} 

\newcommand{\trans}{^{\scriptscriptstyle \top}}

\newcommand{\bproof}{\begin{proof}}
\newcommand{\eproof}{\end{proof}}
\newcommand{\beq}{\begin{equation}}
\newcommand{\eeq}{\end{equation}}
\newcommand{\ext}{{\rm {\cal E}}}
\newcommand{\tr}{{\rm tr}}

\newcommand{\R}{{\mathbb R}}
\newcommand{\G}{{\cal G}}
\newcommand{\V}{{\cal V}}
\newcommand{\WW}{{V}}
\newcommand{\ww}{{v}}

\begin{document}
\bibliographystyle{plainnat}
\makeatletter

\renewcommand\proofname{\bf Proof} 
\def\eop{$\rule{1.3ex}{1.3ex}$}
\renewcommand\qedsymbol\eop  
\makeatletter

\newtheorem{theorem}{Theorem}[section]
\newtheorem{proposition}{Proposition}[section]
\newtheorem{lemma}{Lemma}[section]
\newtheorem*{lemma*}{Lemma}
\newtheorem{corollary}{Corollary}[section]
\newtheorem{definition}{Definition}[section]
\newtheorem{remark}{Remark}[section]

\newcommand{\figsheight}{4.0cm}
\renewcommand\baselinestretch{1}

\begin{titlepage}
\advance\topmargin by 0.5in
\begin{center}

{\Huge New Perspectives on $k$-Support and Cluster Norms}
\vspace{.5truecm}

\end{center}

\begin{center}

{\bf Andrew M. McDonald}, {\bf Massimiliano Pontil}, {\bf Dimitris Stamos}

\vspace{.6truecm}

\noindent 
Department of Computer Science \\
University College London \\
Gower Street, London WC1E \\ England, UK \\
E-mail: {\em \{a.mcdonald, m.pontil, d.stamos\}@cs.ucl.ac.uk}

\vspace{.55truecm}

\end{center}

\begin{abstract}
\noindent 
The $k$-support norm is a regularizer which has been successfully applied to sparse vector prediction problems. 
We show that it belongs to a general class of norms which can be formulated as a parameterized infimum over quadratics.  
We further extend the $k$-support norm to matrices, and we observe that it is a special case of the matrix cluster norm. 
Using this formulation we derive an efficient algorithm to compute the proximity operator of both norms.
This improves upon the standard algorithm for the $k$-support norm and allows us to apply proximal gradient methods to the cluster norm. 
We also describe how to solve regularization problems which employ centered versions of these norms. 
Finally, we apply the matrix regularizers to different matrix completion and multitask learning datasets. 
Our results indicate that the spectral $k$-support norm and the cluster norm give state of the art performance on these problems, significantly outperforming trace norm and elastic net penalties.

\vspace{.3truecm}%
\end{abstract}
\end{titlepage}

\section{Introduction}\label{sec:intro}
We study a family of norms that can be used as regularizers in vector and matrix learning problems.  The norm is obtained by taking an infimum of certain quadratic functions, which are parameterized by a set $\Theta$.  
By varying the set, the regularizer can be tailored to assumptions on the underlying model, which should lead to more accurate learning. 
The norm is defined for $w \in \R^d$, as
\begin{equation}
\Vert w \Vert_{\Theta} = 
\sqrt{ \inf_{\theta \in \Theta} \hspace{.05truecm} \sum_{i=1}^d \frac{w_i^2}{\theta_i} } \label{eqn:theta-primal}
\end{equation}
where $\Theta$ is a convex bounded subset of the positive orthant.   
This family is sufficiently rich to encompasses standard regularizers such as the $\ell_p$ norms \citep{Micchelli2005} for $p \in [1,2]$, the group Lasso \cite{Yuan2006}, Group Lasso with Overlap \citep{Jacob2009}, the norm in \citep{Jacob2009a}, and the structured sparsity norms of \citep{Micchelli2011}. 
Our work builds upon a recent line of papers which considered convex regularizers defined as an infimum problem over a parametric family of quadratics, as well as related infimal convolution problems \citep[see][and references therein]{Jacob2009,Maurer2012,Obozinski2012}. 

In this paper, we focus on the specific case 
\begin{equation}
\Theta =\bigg\{
\theta \in [a,b]^d:\sum_{i=1}^d \theta_i \leq c\bigg\}\label{eqn:theta-abc-def},
\end{equation}
for parameters $0< a < b$ and $c \in [da,db]$, and we refer to the corresponding norm as the box $\Theta$-norm.
As we show, this seemingly simple class nonetheless includes several nontrivial norms such as the recently proposed $k$-support norm \citep{Argyriou2012}.  
This norm is interesting and important as it provides a tight convex relaxation of the cardinality operator on the $\ell_2$ unit ball, and \citep{Argyriou2012} demonstrated that it outperforms the Lasso and the elastic net \citep{Zou2005} on certain data sets. 
We derive an algorithm to compute the proximity operator of the squared box $\Theta$ norm which takes $\mathcal{O}(d \log d)$ time complexity, improving upon the $\mathcal{O}(d(k + \log d))$ complexity in the original work. 

A second important aim of the paper is to extend our formulation to matrix learning regularization, in which the vector norm is applied to the spectrum of a matrix. Following a classical result by von Neumann \citep{VonNeumann1937}, 
we observe that if the set $\Theta$ is invariant under permutations then the formulation \eqref{eqn:theta-primal} can be extended to a orthogonally invariant matrix norm. 
In particular, this allows us to extend the $k$-support norm to a matrix setting. 
In this form we show that it is a special case of the {\em cluster norm} introduced by \citep{Jacob2009a} in the context of clustered multitask learning.  We also describe how to derive the proximity operator of the cluster norm, which enables efficient proximal gradient methods to be applied to this type of matrix learning problems.

In summary, the principal contributions of this paper are:
\begin{itemize}
\item We describe an $O(d \log d)$ method to compute the box $\Theta$-norm and its proximity operator, which allows us to employ optimal first order methods to solve the associated regularization problem;
\item We show that the $k$-support norm is a box $\Theta$-norm and observe that our algorithm to compute the proximity operator improves upon the $\mathcal{O}(d(k+\log d))$ algorithm by \citep{Argyriou2012};
\item  We extend the norm \eqref{eqn:theta-primal} to the matrix setting, showing in particular that the spectral $k$-support spectral norm is a special case of the cluster norm.
\end{itemize}

On the experimental side, we apply these spectral $\Theta$-norms to synthetic and real matrix learning datasets. 
Our findings indicate that the spectral $k$-support norm and the box $\Theta$-norm produce state-of-the art results on the matrix completion benchmarks.  
Furthermore we explain how to perform regularization with centered versions of these norms and we report state of the art results on a multitask learning dataset in this setting. 
Finally we demonstrate numerically the efficiency of our algorithm to compute the proximity operator.

The paper is organized in the following manner.  
In Section \ref{sec:background}, we summarize the properties of the $\Theta$-norm and derive the $k$-support norm formulation.  
In Section \ref{thm:computation-of-theta-norm}, we compute the norm and in Section \ref{sec:computation-of-prox}, we compute the proximity operator of the norm, including the $k$-support norm as a special case. 
In Section \ref{sec:matrix-regularization}, we define the matrix norm and show the cluster norm coincides with the $\Theta$-norm (with a particular choice of parameters $a,b,c$), and in Section \ref{sec:numerics}, we report experiments on simulated and real data sets. 
Finally, in Section \ref{sec:conclusion}, we conclude. 
Derivations that are sketched or missing from the body of the paper are contained in the supplemental material.

\section{Properties of the Norm}
\label{sec:background}
In this section, we highlight some important properties of the class of $\Theta$-norms. We then specialize our observations to the box $\Theta$-norms, showing in particular than it subsumes the $k$-support norm. 

Note that the objective function of problem \eqref{eqn:theta-primal} is strictly convex, hence every minimizing sequence converges to the same point. 
The infimum is, however, not attained in general because a minimizing sequence may converge to a point on the boundary of $\Theta$; for instance, if $\Theta = \{\theta>0 : \sum_{i=1}^d \theta_i \leq 1 \}$, the minimizer is equal to the vector formed by the absolute values of the components of $w$. 
If $\Theta$ is closed, as in the case of \eqref{eqn:theta-abc-def}, then the minimum is attained. 

Our first result establishes that \eqref{eqn:theta-primal} is indeed a norm and derives the form of the dual norm.
\begin{theorem}\label{thm:theta-is-norm}
If $\Theta$ is a convex bounded subset of the positive positive orthant in $\mathbb{R}^d$, then the infimum in equation (\ref{eqn:theta-primal}) defines a norm and the dual norm is given by the formula
\begin{equation}
\Vert u \Vert_{*,\Theta} = \sqrt{  \sup_{\theta \in \Theta}  \hspace{.05truecm}
\sum_{i=1}^d \theta_i u_i^2}. 
\label{eqn:theta-dual}
\end{equation}
\end{theorem}
\begin{proof}(Sketch)\footnote{Full derivations can be found in the appendix.}
The objective in \eqref{eqn:theta-dual} is zero only if $u$ is zero, absolutely homogeneous of order one and convex, so defines a norm.  Using the fact that the Fenchel conjugate of $\frac{1}{2}\Vert \cdot \Vert^2$ is $\frac{1}{2}\Vert \cdot \Vert_*^2$ for any norm \citep[see e.g.][]{Boyd2004}, a direct computation produces a separable optimization problem which can be solved explicitly to recover the form of the primal norm. 
\end{proof}
Note the symmetry between the expressions for the primal norm \eqref{eqn:theta-primal} and the dual norm \eqref{eqn:theta-dual}. 
The latter is obtained by replacing the infimum with the supremum and the terms $1/\theta_i$ by $\theta_i$. Note also that the computation of the dual norm involves searching for a supremum of a linear function of $\theta$ over the bounded convex set $\Theta$. 
It follows that the supremum is achieved at an extreme point of the closure of $\Theta$, that is, denoting by $\ext$ the set of such extreme points, the dual norm can be expressed as 
\begin{equation}
\Vert u \Vert_{*,\Theta} =  \sqrt{ \max_{\theta \in \ext}  \hspace{.05truecm}\sum_{i=1}^d \theta_i u_i^2}.
\label{eq:dual-ext}
\end{equation}
 
For the remainder of the paper, unless otherwise noted, we specialize our observations to the box $\Theta$-norm, that is the case where $\Theta$ is defined as in \eqref{eqn:theta-abc-def}. 
We assume throughout that $0 < a < b \leq 1$, which does not lose any generality since by rescaling and shifting $a$ and $b$ (with corresponding adjustment to $c$) we get an equivalent norm. 
We make the change of variable $\gamma_i = \frac{\theta_i - a}{b-a}$ and observe that the constraints on $\theta$ induce the constraint set $\{\gamma: 0 < \gamma_i \leq 1,~\sum_{i=1}^d \gamma_i \leq \rho\}$, where $\rho = \frac{c - da}{b-a}$. Furthermore
$$
\sum_{i=1}^d \theta_i u_i^2 = a \|u\|_2^2 + (b-a) \sum_{i=1}^d \gamma_i u_i^2.
$$
We conclude that 
\begin{align}
\|u\|_{*,\Theta}^2 = a \|u\|_2^2 + (b-a) \sum_{j=1}^k (|u|_j^{\downarrow})^2 + (\rho - k)|u_{k+1}|
\label{eq:dual-abc}
\end{align}
where $|u|^{\downarrow}$ is the vector obtained from $u$ by reordering its components so that they are non-increasing in absolute value, and $k = \lfloor \rho\rfloor$.
The expression simplifies if $\rho = k$, that is $c=(b-a) k + d a$. 
This is exactly the setting considered by \citep{Jacob2009a}, where $k+1$ is 
interpreted as the number of clusters and $d$ as the number of tasks. 
We return to this interesting case in Section \ref{sec:matrix-regularization}, where we explain in detail how to extend the $\Theta$-norm to a matrix setting.




\subsection{The $k$-Support Norm}

We move on to discuss the connection between the box $\Theta$-norms and the $k$-support norm. 
The latter was motivated by \citep{Argyriou2012} as a tight convex relaxation of the cardinality operator on the $\ell_2$ unit ball, and includes the $\ell_1$-norm and $\ell_2$-norm as special cases. 
Its dual is the $\ell_2$-norm of the largest $k$ components, that is
\beq
\Vert u \Vert_{*,(k)} = \sqrt{\sum_{i=1}^k (\vert u_i \vert^{\downarrow})^2}\label{eqn:k-support-dual}
\eeq
where $k \leq d$. 
Comparing \eqref{eqn:k-support-dual} with \eqref{eq:dual-abc}, we see that the $k$-support norm is obtained as a limiting case of the box $\Theta$-norm for $a\rightarrow 0$,  $b=1$, and $c=k$. 
\begin{theorem}\label{thm:k-support-as-theta}
If $\Theta  = \{\theta \in {\mathbb R}^d: 0 < \theta_i \leq 1,~\sum\limits_{j=1}^d \theta_i \leq k\}$ then $\Vert w \Vert_{\Theta}=\Vert w \Vert_{(k)}$.
\end{theorem}

\citep{Argyriou2012} interpreted the $k$-support norm as a special case of the Group Lasso with Overlap introduced by \citep{Jacob2009}, where the groups have support sets of cardinality at most $k$. The latter norm is defined as the infimal convolution
\beq
\Vert w \Vert_{\G} = \inf_{v \in \V(\G)}  \left\{ \sum_{g \in \G} \Vert v_g \Vert_2 : \sum_{g \in \G} v_g = w \right\}\label{eqn:GLO}
\eeq
where $\mathcal{G}$ is a collection of subsets of $\{1,\dots,d\}$ and 
$$\V(\G) = 
\{(v_g)_{g \in \G}: v_g \in \R^d,~{\rm supp}(v_g) \subseteq g,~g \in \G\}.
$$ 
The $k$-support norm corresponds to the case that $\G$ is formed by all subsets containing at most $k$ elements, that is 
$\G = \G_k := \{g : g \subseteq \{1,...,d\},~|g| \leq k\}$. Furthermore, as noted by \citep{Jacob2009}, the dual norm is given by 
$$
\sqrt{ \max_{g \in {\cal G}} \sum_{i\in g} u_i^2}.
$$
Writing $\sum_{\in g} u_i^2 = \sum_{i=1}^d (1_g)_i u_i^2$, where $(1_g)_i = 1$ if $i \in g$ and zero otherwise, we see that the norm \eqref{eqn:GLO} is a (general) $\Theta$-norm, where $\Theta$ is the convex hull of $ \{1_g : g \in {\cal G}\}$.


We end this section with an infimal convolution interpretation of the box $\Theta$-norm.
\begin{proposition}
If $\Theta$ is of the form in equation \eqref{eqn:theta-abc-def} and $c =(b-a) k + d a$, ~for $k \in \{1, \ldots, d\}$, then
$$
\|w\|_\Theta = \inf \left\{
\sum_{g \in \G_k} 
\sqrt{ \hspace{-.05truecm}\sum_{i \in g} \hspace{-.015truecm}\frac{v_{g,i}^2}{b} \hspace{-.015truecm}+\hspace{-.015truecm}\sum_{i \notin g}   \hspace{-.015truecm}\frac{v_{g,i}^2}{a}} \hspace{-.015truecm}: \hspace{-.015truecm}\sum\limits_{g \in \G_k} v_g = w
\hspace{-.03truecm}\right\}.
$$
\label{prop:infconv-abc}
\end{proposition}
Note that if $b=1$, then as $a$ tends to zero, we obtain the expression of the $k$-support norm \eqref{eqn:GLO} with $\G = \G_k$, recovering in particular the support constraints.
If $a$ is small and positive, the support constraints are not imposed,  
however, 
most of the weight for each $v_g$ tends to be concentrated on $\textrm{supp}(g)$.
Hence, Proposition \ref{prop:infconv-abc} suggests that the box $\Theta$-norm regularizer will encourage vectors $w$ whose dominant components are a subset of a union of a small number of groups $g \in \G_k$.


\section{Computation of the Norm}
\label{sec:computation-of-theta-norm}
In this section, we compute the box $\Theta$-norm by explicitly solving the optimization problem (\ref{eqn:theta-primal}).
\begin{theorem}\label{thm:computation-of-theta-norm}
For $w \in \mathbb{R}^d$, 
\begin{align}
\Vert w \Vert_{\Theta}^2
&= \frac{1}{b^2}  \Vert w_Q \Vert_2^2  + 
\frac{1}{p} \Vert w_I \Vert_1^2 +
\frac{1}{a^2} \Vert w_L \Vert_2^2 \label{eqn:solution-of-abc-norm-2},
\end{align}
where 
$w_Q = (\vert w\vert^{\downarrow}_{1}, \ldots, \vert w\vert^{\downarrow}_{q})$, $x_I = (\vert w\vert^{\downarrow}_{q+1}, \ldots, \vert w\vert^{\downarrow}_{d-\ell})$, $w_L = (\vert w\vert^{\downarrow}_{d-\ell+1}, \ldots, \vert w\vert^{\downarrow}_{d})$, 
$q$ and $\ell$ are the unique integers in $\{0,\ldots, d\}$, satisfying $q+\ell\leq d$ and
\begin{align}
\frac{\vert w_q \vert}{b} &\geq \frac{1}{p}\sum_{i=q+1}^{d-\ell}\vert w_i \vert > \frac{\vert w_{q+1} \vert}{b},\label{eqn:optimal-q}\\
\frac{\vert w_{d-\ell} \vert}{a} &\geq \frac{1}{p}\sum_{i=q+1}^{d-\ell}\vert w_i \vert > \frac{\vert w_{d-\ell+1} \vert}{a}, \label{eqn:optimal-ell}
\end{align}
where $p=c-qb-\ell a$ and we have defined $\vert w_{0} \vert=\infty$ and $\vert w_{d+1} \vert = 0$.
\end{theorem}

\begin{proof}(Sketch)
We need to solve the optimization problem
\beq
\inf_{\theta} \bigg\{ \sum_{i=1}^d \frac{w_i^2}{\theta_i} : a \leq \theta_i \leq b, \sum_{i=1}^d \theta_i \leq c \bigg\}.  
\label{eqn:norm-objective}
\eeq
Without loss of generality we assume that the $w_i$ are ordered non increasing in absolute values, and it follows that at the optimum the $\theta_i$ are also ordered non increasing (see Lemma \ref{lem:ordering}).  We further assume without loss of generality that $w_i \ne 0$ for all $i$ and $c \leq db$, so the sum constraint will be tight at the optimum (see Remark \ref{rem:wi-not-zero}).

The Lagrangian is given by 
\begin{align*}
L(\theta,\alpha) &= \sum_{i=1}^d \frac{w_i^2}{\theta_i}  + \frac{1}{\alpha^2} \left( \sum_{i=1}^d \theta_i-c\right)
\end{align*}
where $1/\alpha^2$ is a strictly positive multiplier to be chosen such that $S(\alpha) := \sum_{i=1}^d \theta_i =c$.
We can then solve the original problem by minimizing the Lagrangian over the constraint $\theta \in [a,b]^d$. Due to the coupling effect of the multiplier, we can solve the simplified problem componentwise using Lemma \ref{lem:mmp-box-solution}, obtaining the solution
\begin{align}
\theta_i = 
\begin{cases}
a, \quad &\text{if} \quad \alpha < \frac{a}{\vert w_i \vert},\\
\alpha \vert w_i \vert, \quad &\text{if} \quad \frac{a}{\vert w_i \vert}\leq \alpha \leq \frac{b}{\vert w_i \vert} ,\\
b, \quad &\text{if}\quad  \alpha > \frac{b}{\vert w_i \vert} ,
\end{cases}\label{eqn:form-of-optimal-theta-abc}
\end{align}
where $S(\alpha) = c$. 
The minimizer has the form $\theta = (b, \ldots, b, \theta_{q+1}, \ldots, \theta_{d-\ell},a, \ldots, a)$, 
where $q,\ell$ are determined by the value of $\alpha$. 
From $S(\alpha) = c$ we get $\alpha = p / (\sum_{i=q+1}^{d-\ell}\vert w_i \vert)$.
The value of the norm in \eqref{eqn:solution-of-abc-norm-2} follows by substituting $\theta$ into the objective. 
Finally, by construction we have $\theta_q \geq b >\theta_{q+1}$ and $\theta_{d-\ell} > a \geq \theta_{d-\ell+l}$, implying \eqref{eqn:optimal-q} and \eqref{eqn:optimal-ell}. 
\end{proof}

Theorem \ref{thm:computation-of-theta-norm} suggests two methods for computing the box $\Theta$-norm.  
First we find $\alpha$ such that $S(\alpha)=c$; this value uniquely determines $\theta$ in \eqref{eqn:form-of-optimal-theta-abc}, and the norm follows by substitution into \eqref{eqn:norm-objective}.
Alternatively we identify $q$ and $\ell$ that jointly satisfy \eqref{eqn:optimal-q} and \eqref{eqn:optimal-ell} and we compute the norm using  \eqref{eqn:solution-of-abc-norm-2}.  
Taking advantage of the structure of $\theta$ we now show that the first method leads to a computation time that is $\mathcal{O}( d \log d)$.

\begin{theorem}\label{thm:theta-norm-d-log-d}
The computation of the $\Theta$-norm can be completed in $\mathcal{O}( d \log d)$ time. 
\end{theorem}

\begin{proof}
Following Theorem \ref{thm:computation-of-theta-norm}, we need to choose $\alpha^*$ to satisfy the coupling constraint $S(\alpha^*)  = c$.
Each component of $\theta$ is a piecewise linear function in the form of a step function with a constant slope between the values $a$ and $b$.  
Let the set $\left\{ \alpha^{i} \right\}_{i=1}^{2d}$ be the set of the $2d$ critical points, where the $\alpha^{i}$ are ordered non decreasing.  
The function $S(\alpha)$ is a non decreasing piecewise linear function with at most $2d$ critical points.
 
We can find $\alpha^*$ by first sorting the points $\{ \alpha^i \}$, finding $\alpha^i$ and $\alpha^{i+1}$ such that $S(\alpha^i) \leq c \leq S(\alpha^{i+1})$ by binary search, and then interpolating $\alpha^*$ between the two points. 

Sorting takes $\mathcal{O}(d \, \log d)$.  
Computing $S(\alpha^i)$ at each step of the binary search is $\mathcal{O}(d )$, so $\mathcal{O}(d \, \log d)$ overall.  
Given $\alpha^i$ and $\alpha^{i+1}$, interpolating $\alpha^*$ is $\mathcal{O}(d)$,  so the algorithm overall is $\mathcal{O}(d \, \log d)$ as claimed.
\end{proof}

As outlined earlier, the $k$-support norm is a special case of the box $\Theta$-norm.  
We have the following corollary of Theorem \ref{thm:computation-of-theta-norm} and Theorem \ref{thm:theta-norm-d-log-d}.

\begin{corollary}
For $w \in \mathbb{R}^d$, and $k\leq d$, 
\begin{align*}
\Vert w \Vert_{(k)} &=  \Bigg( \sum_{j=1}^q \vert w^{\downarrow}_j \vert^2  + 
\frac{1}{k-q} \Big(\sum_{j=q+1}^d \vert w^{\downarrow}_j \vert\Big)^2 \Bigg)^{\frac{1}{2}},
\end{align*}
where $q$ is the unique integer in $\{0, k-1\}$ satisfying
\begin{align}
\vert w_{q} \vert \geq \frac{1}{k-q} \sum_{j=q+1}^d \vert w_{j} \vert > \vert w_{q+1} \vert, \label{eqn:k-sup-optimal-q}
\end{align}
and we have defined $w_0=\infty$.  Furthermore, the norm can be computed in $\mathcal{O}(d \log d)$ time. 
\end{corollary}

\begin{proof}
The first claim follows from Theorem \ref{thm:computation-of-theta-norm} by letting $a \rightarrow 0$, $b=1$, $c=k$ and noting that equation \eqref{eqn:optimal-ell} is no longer required.  The complexity follows from Theorem \ref{thm:theta-norm-d-log-d}.
\end{proof}


\section{Proximity Operator}\label{sec:computation-of-prox}

Proximal gradient methods can be used to efficiently solve optimization problems of the form $\min_{w \in \mathbb{R}^d} f(w) + \lambda g(w)$, 
where $f$ is a convex loss function with Lipschitz continuous gradient, $\lambda >0$ is a regularization parameter, and $g$ is a convex function for which the proximity operator can be computed efficiently \citep[see][and references therein]{Combettes2010, Beck2009, Nesterov2007}. 
The proximity operator of $g$ with parameter $\rho>0$ is defined as 
\begin{align*}
\prox_{\rho g} (w) = \argmin_{x \in\R^d} \left\{\frac{1}{2}\Vert x-w\Vert^2 + \rho g(x)\right\}.
\end{align*}
We now use the infimum formulation of the box $\Theta$-norm to derive the proximity operator of the squared norm.

\begin{theorem}\label{thm:prox-of-general-theta-norm}
The proximity operator of the square of the box $\Theta$-norm at point $w \in \R^d$ with parameter $\frac{\lambda}{2}$ is given by
$\prox_{\frac{\lambda}{2}\Vert \cdot \Vert_{\Theta}^2}(w) = (\frac{\theta_1 w_1}{\theta_1+\lambda},\dots, \frac{\theta_d w_d}{\theta_d+\lambda})$, where
\begin{align}
\theta_i &= 
\begin{cases}
    		a, & \text{if} \, \alpha  < \frac{a+\lambda}{\vert w_i \vert}, \\
    		\alpha \vert w_i \vert - \lambda, & \text{if} \,  
    		\frac{a+\lambda}{\vert w_i \vert}\leq \alpha \leq \frac{b+\lambda}{\vert w_i \vert}\\
    		b, & \text{if} \, \alpha  > \frac{b+\lambda}{\vert w_i \vert} ,
  	\end{cases} \label{eqn:cluster-theta}
\end{align}
and $\alpha$ is chosen such that $S(\alpha) = \sum_{i=1}^d \theta_i = c$. 
\end{theorem}

\begin{proof}
Similarly to Remark \ref{rem:wi-not-zero} for Theorem \ref{thm:computation-of-theta-norm}, we can assume without loss of generality that $w_i\ne 0$. 
Using the infimum formulation of the norm, and noting that the infimum is attained, we solve
\begin{align*}
\min_{x\in \mathbb{R}^d} \min_{\theta \in \Theta} \, \left\{\frac{\lambda}{2}\sum_{i=1}^d \frac{x_i^2}{\theta_i} + \frac{1}{2}\sum_{i=1}^d (x_i-w_i)^2 \right\}.
\end{align*}
We can exchange the order of the optimization and solve for $x$ first.  
The problem is separable we find $x_i = \frac{\theta_i w_i}{\theta_i + \lambda}$. 
Discarding a multiplicative factor of $\lambda/2$, the problem in $\theta$ becomes
$$
\min_{\theta} \bigg\{ 
\sum_{i=1}^d \frac{ w_i^2}{\theta_i+\lambda} : {a \leq \theta_i \leq b, \sum_{i=1}^d \theta_i \leq c} \,\bigg\}
$$
and we can equivalently solve  
$$
\min_{a \leq \theta_i \leq b} 
\left\{\sum_{i=1}^d \frac{ w_i^2}{\theta_i+\lambda} + \frac{1}{\alpha^2}  \left( \sum_{i=1}^d \theta_i - c \right)\right\},
$$
where $\frac{1}{\alpha^2}$ is a strictly positive Lagrange multiplier and $\alpha$ is chosen to ensure that $S(\alpha)=c$.  
The problem is separable and decouples, 
and by Lemma \ref{lem:mmp-box-solution} the solution is given by
\begin{align*}
\theta_i &= 
\begin{cases}
    		a, & \text{if} \, \alpha  < \frac{a+\lambda}{\vert w_i \vert}, \\
    		\alpha \vert w_i \vert - \lambda, & \text{if} \,  
    		\frac{a+\lambda}{\vert w_i \vert}\leq \alpha \leq \frac{b+\lambda}{\vert w_i \vert}\\
    		b, & \text{if} \, \alpha  > \frac{b+\lambda}{\vert w_i \vert} ,
  	\end{cases}
\end{align*}
where $\alpha$ is chosen such that $S(\alpha) = c$. 
\end{proof}

By the same reasoning as for Theorem \ref{thm:theta-norm-d-log-d} we have the following. 
\begin{theorem}\label{thm:prox-d-log-d}
The computation of the proximity operator can be completed in $\mathcal{O}( d \log d)$ time. 
\end{theorem}

Algorithm \ref{alg:prox_01} illustrates the computation of the proximity operator for the squared box $\Theta$-norm in $\mathcal{O}(d \log d)$ time.  This includes the $k$-support as a special case, which improves upon the complexity of the $\mathcal{O}(d(k + \log d))$ computation provided in \citep{Argyriou2012}.  
We summarize this in the following corollary.

\begin{corollary}\label{cor:prox-of-ksup}
The proximity operator of the square of the $k$-support norm at point $w$ with parameter $\frac{\lambda}{2}$ is given by
$\prox_{\frac{\lambda}{2}\Vert \cdot \Vert_{\Theta}^2}(w) = x$, 
where $x_i =  \frac{\theta_i w_i}{\theta_i+\lambda}$, and
\begin{align*}
\theta_i &= 
\begin{cases}
    		0, & \text{if} \, \alpha  < \frac{\lambda}{\vert w_i \vert}, \\
    		\alpha \vert w_i \vert - \lambda, & \text{if} \,  
    		\frac{\lambda}{\vert w_i \vert}\leq \alpha \leq \frac{1+\lambda}{\vert w_i \vert}\\
    		1, & \text{if} \, \alpha  > \frac{1+\lambda}{\vert w_i \vert} ,
  	\end{cases}
\end{align*}
where $\alpha$ is chosen such that $S(\alpha) = k$.
Furthermore, the proximity operator can be computed in $\mathcal{O}(d \log d)$ time. 
\end{corollary}

\begin{proof}
The expression follows directly from Theorem \ref{thm:prox-of-general-theta-norm} by letting $a\rightarrow 0$, and setting $b=1$ and $c=k$ in equation (\ref{eqn:cluster-theta}) and the complexity claim follows directly from Theorem \ref{thm:prox-d-log-d}.
\end{proof}

\begin{algorithm}[ht]
\caption{Computation of $x= \prox_{\frac{\lambda^2}{2}\|\cdot \|_{\Theta}^2}\left(w\right)$. \label{alg:prox_01}}
\begin{algorithmic} 
\REQUIRE parameters  $a$, $b$, $c$, $\lambda$.
\STATE \textbf{1.} Sort points $\left\{ \alpha^i \right\}_{i=1}^{d} = \left\{ \frac{a+\lambda}{ \vert w_j \vert}, \frac{b+\lambda}{ \vert w_j \vert} \right\}_{j=1}^d$  such that $\alpha^i \leq \alpha^{i+1}$.
\STATE \textbf{2.} Identify points $\alpha^i$ and $\alpha^{i+1}$ such that $S(\alpha^i) \leq c$ and $S(\alpha^{i+1})\geq c$ by binary search.
\STATE \textbf{3.} Find $\alpha^*$ between $\alpha^i$ and $\alpha^{i+1}$ such that $S(\alpha^*)=c$ by linear interpolation.
\STATE \textbf{4.} Compute $\theta_i(\alpha^*)$ for $i=1\ldots, d$.
\STATE \textbf{5.} Return $x_i =\frac{\theta_i w_i}{\theta_i+\lambda}$ for $i=1\ldots, d$.
\end{algorithmic}
\end{algorithm}


\section{Spectral Regularization}\label{sec:matrix-regularization}
In this section, we use the box $\Theta$-norm to define a class of orthogonally invariant matrix norms. In particular, we extend the $k$-support norm to this setting and note that it is closely related to the cluster norm of \citep{Jacob2009a}.

A norm $\Vert\cdot \Vert$ on $\mathbb{R}^{d \times m}$ is called orthogonally invariant if 
\begin{align}
\Vert W \Vert= \Vert U W V\trans \Vert, \label{eqn:unitarily-invariant-def} 
\end{align}
for any orthogonal matrices $U \in \mathbb{R}^{d \times d}$ and $V \in \mathbb{R}^{m \times m}$. 
A classical result of von Neumann \citep{VonNeumann1937} establishes that a norm is orthogonal invariant if and only if 
it is of the form $\Vert W \Vert = g(\sigma(W))$, where $\sigma(W)$ is the vector formed by the singular values of $W$, and $g$ is a {\em symmetric gauge function}.  This means that $g$ is a norm which is invariant under permutations and sign changes. 
That is, $g(w_1,...,w_d) = g(w_{\pi(1)},\ldots,w_{\pi(d)})$ for every permutation $\pi$, and $g(Jw) = g(w)$ for every diagonal matrix $J$ with entries $\pm 1$.  

Considering the box $\Theta$-norm objective in equation (\ref{eqn:theta-primal}), we note that it is sign invariant. 
Moreover, if the set $\Theta$ is permutation invariant (meaning that if $\theta \in \Theta$ then $\theta_\pi \in \Theta$ for every permutation $\pi$), then the $\Theta$-norm is permutation invariant.
This holds for the set \eqref{eqn:theta-abc-def} as the components $\theta_i$ each lie in the same interval $[a,b]$, and the sum is invariant to the order of the terms by Lemma \ref{lem:ordering}, we conclude that the corresponding box $\Theta$-norm is a symmetric gauge function. 
By applying this norm to the spectrum of a matrix we obtain an orthogonally invariant norm which we term a \emph{spectral $\Theta$-norm}. 
In particular choosing $a=0,b=1,c=k$, we obtain the spectral $k$-support norm, which we define (with some abuse of notation) as
$\Vert W \Vert_{(k)} = \Vert \sigma(W) \Vert_{(k)}$.


\subsection{Cluster Norm for Multi Task Learning}
We now briefly discuss multitask learning. 
The problem is given by 
\begin{align*}
\min_{W \in \mathbb{R}^{d \times m}} \ell(W) + \lambda \Omega(W)
\end{align*} 
for some loss function $\ell$ and regularizer $\Omega$. 
The columns of $W$ represent different regression vectors (tasks). 
A natural assumption that arises in applications is that the tasks are clustered and a matrix regularizer can be chosen to favour such structure. 
We refer to \citep{Evgeniou2005,Argyriou2008,Jacob2009a} for a discussion and motivating examples.

In particular, \citep{Jacob2009a} propose a composite regularizer which balances a tradeoff 
between the norm of the mean of the tasks, 
the variance between the clusters, and the variance within the clusters and define  
the cluster norm as 
\beq
\Vert W \Vert_{\rm cl} = \sqrt{\inf_{\Sigma \in \mathcal{S}} \tr (W \Sigma^{-1} W\trans)}
\label{eq:CN}
\eeq
where $\mathcal{S}$ is a set of symmetric positive definite matrices subject to a spectral constraint, 
\beq
\mathcal{S} = \{ \Sigma \succeq 0 : a I \preceq \Sigma \preceq b I,~ \tr \, \Sigma = c \}
\label{eq:CNset}
\eeq
and $c=(b-a) k + m a$. Here $k+1$ plays the role of the number of clusters and the parameters $a$ and $b$ control 
the variances within and between the clusters respectively \citep[see][for a discussion]{Jacob2009a}.

The cluster norm corresponds to the matrix norm where we apply the box $\Theta$-norm to the singular values of the matrix.  
\begin{proposition}\label{thm:cluster-quadratic-form}
The cluster norm \eqref{eq:CN} is an orthogonally invariant norm and can be written as
\begin{align}
\Vert W \Vert_{\rm cl} = \Vert \sigma(W) \Vert_{\Theta} = \sqrt{\inf_{\theta \in \Theta} \sum_{i=1}^d \frac{\sigma_i(W)^2}{\theta_i}}, \label{eqn:theta-form-of-cluster-norm}
\end{align}
where $\Theta =\left\{ \theta \in \mathbb{R}^d:  a \leq \theta \leq b , \sum_{i=1}^d \theta_i \leq c \right\}$.
\end{proposition}

\begin{proof}
Let $(U,\Sigma,V)$ be the singular value decomposition of $W$ and note that
\begin{align*}
\tr (W E^{-1} W\trans) & = \tr (U \Sigma V\trans E^{-1} V \Sigma\trans U\trans) \\
& =\tr (\Sigma V\trans E^{-1} V \Sigma\trans),
\end{align*}
where the second equality follows by cyclic property of the trace operator. 
Now let $E' = V\trans E^{-1} V$ and observe that 
$\sigma(E) = \sigma(E')$. We conclude that the cluster norm is orthogonally invariant and can be expressed in the form \eqref{eqn:theta-form-of-cluster-norm}.
\end{proof}

The computational considerations in Sections \ref{sec:computation-of-theta-norm} and \ref{sec:computation-of-prox} can be naturally extended to the matrix setting by using von Neumann's trace inequality \citep[see e.g.][ex. 3.3.10]{Horn1991}. 
Here we only comment on the computation of the proximity operator which is important for our numerical experiments below. 
The proximity operator of an orthogonally invariant norm $\Vert \cdot \Vert = g( \sigma(\cdot))$ is given by
\begin{align*}
\prox_{\Vert \cdot \Vert} (W)= U \text{diag}(\prox_{\|\cdot\|}(\sigma(W))) V\trans,
\end{align*}
where $U$ and $V$ are the matrices formed by the left and right singular vectors of $W$ \citep[see, e.g.][Prop 3.1]{Argyriou2011}. 
Using this result we can apply any spectral $\Theta$-norm, such as the spectral $k$-support norm, to matrix learning problems as we describe in Section \ref{sec:numerics} below.


\subsection{Centered Cluster Norm}

In their work, \citep{Jacob2009a} consider the following regularization problem:
\begin{align}
\min_{W \in \mathbb{R}^{d \times m}}  \ell_{\rm c}(W) + \lambda \Vert W \Pi \Vert_{\rm cl}^2 \label{eqn:cluster-norm-problem}.
\end{align}
where $\Pi = I - \frac{11\trans}{m~}$, is the centering operator, so that $\Pi W = [w_1- {\bar w},\dots,w_m - {\bar w}]$ with ${\bar w} = \frac{1}{m} \sum_{i=1}^m w_i$ and $\ell_{\rm c}$ is a modified loss term which incorporates the square norm regularization of the mean of the tasks. We end this section by discussing how to solve the optimization problem \eqref{eqn:cluster-norm-problem}. The following lemma is key.

\begin{lemma}
\label{lem:infimum-of-cluster}
For all $W =[w_1,\dots,w_m] \in {\mathbb R}^{d \times m}$, it holds 
\begin{align*}
\|W\Pi\|_{\rm cl} = \min_{z \in {\mathbb R}^d} \|[w_1-z,\cdots,w_m-z]\|_{\rm cl}.
\end{align*}
\end{lemma}

\begin{proof} Since $\|\cdot\|_{\rm cl}$ is orthogonally invariant
\begin{align*}
\Vert W \Pi \Vert_{\rm cl}^2
&= \Vert (W \Pi)\trans \Vert_{\rm cl}^2 \\
&= \inf_{\Sigma \in \mathcal{S}_c} \textrm{tr} ( ( W \Pi)\trans \Sigma^{-1} (W \Pi )) \\
&= \inf_{\Sigma \in \mathcal{S}_c} \sum_{i=1}^m (w_i- \bar{w})\trans \Sigma^{-1} (w_i - \bar{w}) \\
&= \inf_{\Sigma \in \mathcal{S}_c} \min_{z \in \mathbb{R}^d} \sum_{i=1}^m (w_i- z)\trans \Sigma^{-1} (w_i - z) \\
&=\min_{z \in \mathbb{R}^d} \|[w_1- z,\dots,w_m-z]\|_{\rm cl},
\end{align*}
where we have used the fact that $W\Pi$ is the centred weight matrix $[w_1-\bar{w}, \ldots, w_m-\bar{w}]$, with $\bar{w} = \frac{1}{m} \sum_{i=1}^m w_i$, and the fact that for $U \in \mathbb{R}^{p \times m}$, $S  \in \mathbb{R}^{p \times q} $, $V  \in \mathbb{R}^{q \times m}$, where $U, V$ have columns $u_i$, $v_i$ respectively, we have 
\begin{align*}
\textrm{tr}(U\trans S V) = \sum_{i=1}^m u_i\trans S v_i .
\end{align*}
Finally, note that the quadratic $\sum_{i=1}^m (w_i- z)\trans \Sigma^{-1} (w_i - z)$ is minimized at $z=\bar{w}$.
\end{proof}

Using this lemma, and letting $\ww_i = w_i -z$, and $\WW = [w_1-z, \ldots, w_m-z]$, we rewrite problem \eqref{eqn:cluster-norm-problem} as
the equivalent problem in $m(d+1)$ variables
\begin{align}
\min_{(\WW,z) \in \mathbb{R}^{d \times m} \times  \mathbb{R}^d}  
\ell_c([\ww_1+z, \ldots, \ww_m+z]) + \lambda \Vert \WW \Vert_{\rm cl}^2.\label{eqn:cluster-norm-optimization}
\end{align}
This problem is of the form $\ell(V,z) + f(V,z)$, where $f(V,z) = \lambda \|V\|_{\rm cl}$. 
Using this formulation, we can directly apply the proximal gradient method using the proximity operator computation for the cluster norm, since $\prox_f(V_0,z_0) = (\prox_{\lambda \|\cdot\|_{\rm cl}}(V_0),z_0)$. 
Finally we point out that this method can be used to perform optimization with the centered trace norm and the centered spectral $k$-support norm since both can be written as a cluster norm.


\section{Numerical Experiments}\label{sec:numerics}
In this section, we report on the statistical performance of the spectral $\Theta$-norms in regularization problems, and we briefly comment on the numerical performance of Algorithm \ref{alg:prox_01} for the $k$-support norm.
In \citep{Argyriou2012}, the authors demonstrated the good estimation properties of the vector $k$-support norm compared to the Lasso and the elastic net.   
Here, we investigate the spectral $\Theta$-norms including the spectral $k$-support norm.
For both matrix completion and multitask learning we consider the spectral $k$-support (\emph{ks}) and the spectral $\Theta$-norm (\emph{box}).  
In the latter case we also consider the centered cluster norm (\emph{c-cn}) and the centered spectral $k$-support norm (\emph{c-ks}).
For both set of experiments we included the trace norm \citep[see, e.g.][]{Cai2008} and the spectral elastic net.

All optimizations used an accelerated proximal gradient method (FISTA), \citep[see e.g.][]{Beck2009, Combettes2010, Nesterov2007}.  
\footnote{Code used in the experiments is available at 
http://www0.cs.ucl.ac.uk/staff/M.Pontil/software.html.}
We report on test errors, number of iterations to convergence ($N$), final matrix rank ($r$) and parameter values for $k$ and $a$, all of which are chosen by validation along with the regularization parameter.

\subsection{Proximity Operator} 
We first verified the improved performance of the proximity operator computation. 
Table \ref{table:prox_comparison} illustrates the time taken by the $\mathcal{O}(d (k+\log d))$ method proposed in \citep{Argyriou2012} and by our $\mathcal{O}(d \log d)$ method on a synthetic noisy dataset, fixing $k=d/100$, and confirms the theoretical improvement.

\begin{table}
\caption{Comparison of proximity operator algorithms for $k$-support norm (in s), $k = d/100$. Algorithm 1 is the method in \citep{Argyriou2012}, Algorithm 2 is our method.}
\label{table:prox_comparison}
\vskip 0.15in
\begin{center}
\begin{small}
\begin{tabular}{|l|c|c|c|c|c|}
\hline
$k$ &   10 &  20 &  40 & 80 &  160 \\ \cline{1-6}
\hline
\hline
Alg. 1 & 0.0101  &  0.0361 &   0.1339 &   0.5176 &   2.0298 \\ \cline{1-6}
Alg. 2 & 0.0012  &  0.0017 &   0.0028 &   0.0044 &   0.0096 \\ \cline{1-6}
\hline
\end{tabular}
\end{small}
\end{center}
\vskip -0.1in
\end{table}

\subsection{Simulated Data}
{\bf Matrix completion.} We applied the regularizer to matrix completion on noisy observations of low rank matrices. 
Each $m \times m$ matrix is generated as $W=U V^T+\mathcal{E}$, where $U,V \in \mathbb{R}^{m\times r}$, $r \ll m$, and the entries of $U$, $V$ and $\mathcal{E}$ are i.i.d. standard Gaussian.

We sampled uniformly a percentage $\rho$ of the entries for training, of which 10\% were used in validation. 
The error was measured as $\frac{\|\text{true} - \text{predicted}\|^2}{\|\text{true}\|^2}$ \citep{Mazumder2010} and averaged over 100 trials.
The convergence criterion was the percentage change in objective, with a tolerance of $10^{-5}$.
The results are summarized in Table \ref{table:matrix_completion_synthetic}.

In all training regimes, the spectral $\Theta$-norm generated the lowest test errors, and converged within the fewest iterations.  
While it generated a matrix with the highest rank, the spectrum shows a rapid decrease after the first few singular values, which suggests a threshholding approach might be beneficial.
In the absence of noise, the spectral $\Theta$-norm performed worse than the trace and spectral $k$-support norms, which suggests that the lower bound on the non zero singular values due to $a$ counteracted the noise in this instance.

\begin{table}
\caption{Matrix completion on simulated data sets. }
\label{table:matrix_completion_synthetic}
\vskip 0.15in
\begin{center}
\begin{small}
\begin{tabular}{|l|l|c|c|c|c|c|}
\hline
   dataset & norm  & test error              & $N$       & $r$          & $k$ & $a$  \\
\hline
\hline
rank 5       & tr  & 0.8184 (0.03)           & 77          & 20           & -   &-  \\ \cline{2-7} 
$\rho$=10\%  & en  & 0.8164 (0.03)           & 65          & 20           &  -  & - \\ \cline{2-7} 
             & ks  & 0.8036 (0.03)           & 39          & 16           & 3.6 & -\\ \cline{2-7}
             & box & \textbf{0.7805 (0.03)}  & 33          & 87           & 2.9 & 0.017   \\ 
\hline
rank 5       & tr    & 0.5764 (0.04)           & 60          & 22   &-        & -     \\ \cline{2-7} 
$\rho$=15\%  & en  & 0.5744 (0.04)           & 56          & 21     &-      & -     \\ \cline{2-7} 
             & ks   & 0.5659 (0.03)           & 42          & 18     & 3.3    &- \\ \cline{2-7}
             & box      & \textbf{0.5525 (0.04) } & 30          & 100   & 1.3 & 0.009  \\ 
\hline
rank 5       & tr  & 0.4085 (0.03)           & 47          & 23           &  -  & - \\ \cline{2-7} 
$\rho$=20\%  & en  & 0.4081 (0.03)           & 45          & 23           &  -  & - \\ \cline{2-7} 
             & ks  & 0.4031 (0.03)           & 33          & 21           & 3.1 & - \\ \cline{2-7}
             & box & \textbf{0.3898 (0.03) } & 25          & 100          & 1.3 & 0.009     \\ 
\hline
\hline
rank 10      & tr  & 0.6356 (0.03)           & 49          & 27           &-  & -   \\ \cline{2-7} 
$\rho$=20\%  & en  & 0.6359 (0.03)           & 48          & 27           & - & -   \\ \cline{2-7} 
             & ks  & 0.6284 (0.03)           & 34          & 24           & 4.4 &-  \\ \cline{2-7}
             & box & \textbf{0.6243 (0.03)}  & 25          & 89           & 1.8  & 0.009      \\ 
\hline
rank 10      & tr  & 0.3642 (0.02)           & 68          & 36           &  -  & - \\ \cline{2-7} 
$\rho$=30\%  & en  & 0.3638 (0.02)           & 61          & 36           &  -  & - \\ \cline{2-7} 
             & ks  & 0.3579 (0.02)           & 42          & 33           & 5.0 &- \\ \cline{2-7}
             & box & \textbf{0.3486 (0.02)}  & 30          & 100           & 2.5 & 0.009      \\ 
\hline
\end{tabular}
\end{small}
\end{center}
\vskip -0.1in
\end{table}

{\bf Clustered Learning.} As a proof of concept we tested the centered cluster norm and the centered $k$-support norm on an synthetic dataset.
We generated a $100 \times 100$, rank 5, block diagonal matrix, where the entries of each $20\times 20$ block were equal to a given constant (between 1 and 10) plus noise.  
Table \ref{table:simulated-block-diagonal} illustrates the results for different training set sizes averaged over 100 runs.  
The cluster norm performed the best, however it generated full rank matrices (with the characteristic rapid spectral decay).
Furthermore, the centered spectral $k$-support norm was a close second, with test errors converging towards those of the cluster norm as the training set increased in size, while maintaining a low rank due to the sparsity-inducing properties of the $k$-support norm.
Figure \ref{fig:block-matrix} illustrates a sample matrix along with the final matrix determined by the cluster and trace norms.

\begin{figure}
\vskip 0.2in
\begin{center}
\centerline{\includegraphics[width=0.45\columnwidth]{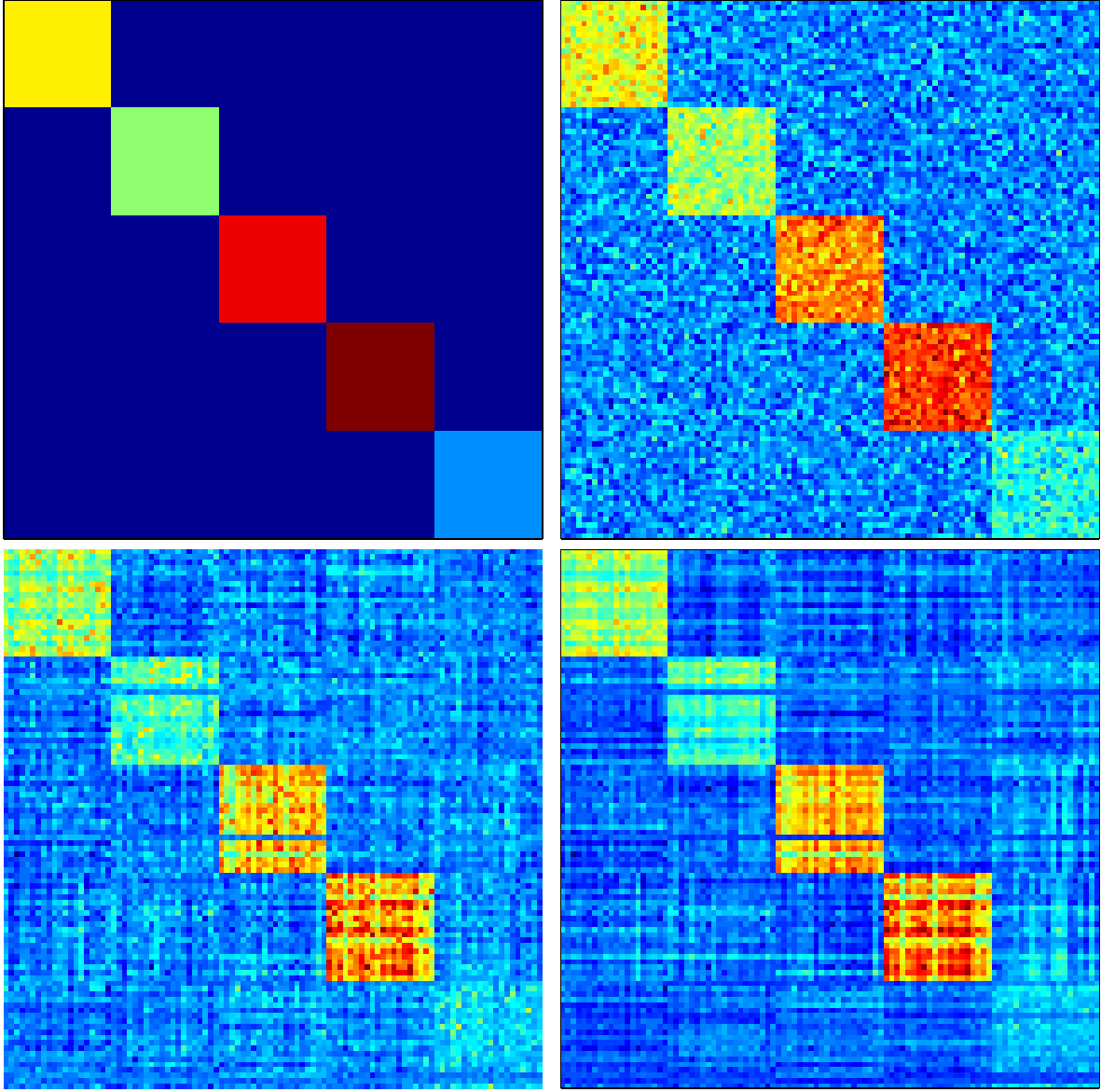}}
\caption{Block clustered matrix and recovered solution. Clockwise from top left: true, noisy, centered cluster norm, trace norm. }
\label{fig:block-matrix}
\end{center}
\vskip -0.2in
\end{figure}

\begin{table}
\caption{Clustered block diagonal matrix.}
\label{table:simulated-block-diagonal}
\vskip 0.15in
\begin{center}
\begin{small}
\begin{tabular}{|l|l|c|c|c|c|}
\hline
dataset & norm & test error    & $r$ & $k$ & $a$          \\ 
\hline
\hline

$\rho$=10\%    & tr         & 0.6909 (0.07) & 23 &        -  & -           \\ \cline{2-6} 
               & en         & 0.6876 (0.08)  & 22&         -  & -          \\ \cline{2-6} 
              & ks         & 0.6740 (0.08) & 21 & 3.10        & -         \\ \cline{2-6} 
                & box       & 0.6502 (0.07) & 100 & 2.86 & 0.0052   \\ \cline{2-6} 
                & c-ks       & \textbf{0.6211 (0.05)} & 16 & 2.21        & -         \\ \cline{2-6} 
                & c-cn        & \textbf{0.6085 (0.05)} & 100& 2.28 & 0.002    \\                 
\hline
\hline
$\rho$=15\%    & tr         & 0.3845 (0.05) & 22 &         -   &-          \\ \cline{2-6} 
               & en         & 0.3831 (0.05) & 22&          -  &-           \\ \cline{2-6} 
               & ks         & 0.3820 (0.05)  & 22& 2.70       &-          \\ \cline{2-6} 
               & abc        & 0.3798 (0.05) & 100&  2.48 & 0.004     \\ \cline{2-6} 
                & c-ks       & \textbf{0.3495 (0.04)} & 19 & 2.01        & -         \\ \cline{2-6} 
                & c-cn        & \textbf{0.3446 (0.03)} & 100 & 2.10 & 0.001  \\ 
\hline
\hline
$\rho$=20\%   & tr          & 0.2596 (0.03)& 25  &       -  &-             \\ \cline{2-6} 
              & en          & 0.2592 (0.03)  & 25 &        -  &-           \\ \cline{2-6} 
             & ks          & 0.2576 (0.03) & 23 & 2.80        &-         \\ \cline{2-6} 
              & box         & 0.2516 (0.03)  & 100 & 2.53 & 0.006   \\ \cline{2-6} 
              & c-ks         & \textbf{0.2228 (0.02)} & 22 & 2.26        & -         \\ \cline{2-6} 
              & c-cn          & \textbf{0.2221 (0.02)} & 100 & 2.42 & 0.002   \\ 
\hline
\end{tabular}
\end{small}
\end{center}
\vskip -0.1in
\end{table}

\subsection{Matrix Completion}
Next we considered different matrix completion data sets, where a percentage of the (user,rating) entries of a matrix are observed, and the task is to predict the unobserved ratings, with the assumption that the true underlying matrix is low rank.
We considered the following datasets.

\emph{MovieLens 100k.}\footnote{http://grouplens.org/datasets/movielens/} This dataset consists of 943 users and 1,682 movies, with a total of 100,000 ratings from $1$ to $5$.  

\emph{Jester-1\footnote{http://goldberg.berkeley.edu/jester-data/}.} This dataset consists of 24,983 users and 100 jokes, all users have rated a minimum of 36 jokes and ratings are real values from -10 to 10.   

\emph{Jester-2}. This dataset consists of 23,500 users and 100 jokes, all users have rated a minimum of 36 jokes and ratings are real values from -10 to 10.

\emph{Jester-3.} This dataset consists of 24,938 users and 100 jokes, all users have rated between 15 and 35 jokes and ratings are real values from -10 to 10. 

For Movielens we uniformly sampled $\rho=50\%$ of the available entries for each user for training, as in \citep{Toh2011}, for Jester-1 and Jester-2, 20 entries, and for Jester-3, 8 entries, and we used 10\% of the these for validation.
The error measurement was Normalized Mean Absolute Error as in \citep{Toh2011}, 
\begin{align*}
\text{NMAE} &= \frac{\Vert \text{true}-\text{predicted}\Vert^2}{\frac{\#\text{observations}}{r_{\max} -r_{\min}}}.
\end{align*} 
and we used percentage change in objective, with a tolerance of $10^{-3}$, averaged over 50 runs.  
The results are outlined in Table \ref{table:matrix_completion_real}.  
Note that on all three datasets, the spectral $k$-support performed the best (standard deviations were of the order of $10^{-5}$), and we note that our results for trace norm regularization on MovieLens agreed with \citep{Jaggi2010}.

\begin{table}
\caption{Matrix completion on MovieLens and Jester data sets. 
}
\label{table:matrix_completion_real}
\vskip 0.15in
\begin{center}
\begin{small}
\begin{tabular}{|l|l|c|c|c|c|c|}
\hline
 dataset   & norm & test error      & $N$       & $r$          & $k$ & $a$ \\
\hline
\hline
Movie-      & tr   & 0.2034          & 57          & 87           &   -&-        \\ \cline{2-7} 
Lens        & en & 0.2034          & 67          & 87           &    -&-       \\ \cline{2-7} 
$\rho=50\%$ & ks  & \textbf{0.2031} & 67          & 102          & 1.00   &-  \\ \cline{2-7} 
            & box     & 0.2035           & 61          & 943         & 1.00 & $10^{-5}$ \\               
\hline
\hline              
Jester1     & tr   & 0.1787          & 14          & 98           &    -&-      \\ \cline{2-7} 
20 per      & en & 0.1787          & 14          & 98           &   -&-       \\ \cline{2-7} 
  line      & ks  & \textbf{0.1764} & 14          & 84           & 5.00 &-    \\ \cline{2-7}
       	    & box     & 0.1766       & 16           & 100      & 4.00 & $10^{-6}$  \\        	
\hline

\hline              
Jester2     & tr   &  0.1788           & 14          & 97           &    -&-      \\ \cline{2-7} 
20 per      & en &  0.1788         & 14          & 97           &   -&-       \\ \cline{2-7} 
  line      & ks  & \textbf{0.1767} & 13          & 75           & 5.00 &-    \\ \cline{2-7}
       	    & box     & 0.1767       & 14           & 100      & 4.00 & $10^{-6}$  \\        	
\hline
\hline            
Jester3     & tr   & 0.1988          & 16          & 49           &   -&-       \\ \cline{2-7} 
8  per      & en & 0.1988          & 16          & 49           &   -&-       \\ \cline{2-7} 
  line      & ks  & \textbf{0.1970} & 17          & 46           & 3.70 &-   \\ \cline{2-7}
            & box     & 0.1973       & 23           &  100          & 5.91 & 0.001 \\ 
\hline
\end{tabular}
\end{small}
\end{center}
\vskip -0.1in
\end{table}

\subsection{Clustered Multitask Learning}
In our final set of experiments we considered the \emph{Lenk personal computer} dataset \citep{Lenk1996}, see also \citep{Argyriou2008}.  
This consists of 180 ratings on a scale from 0 to 10 of 20 profiles of computers characterized by 13 features ($d=14$ with bias term).  
The clustering is suggested by the assumption that users are motivated by similar groups of features. 
We used the RMSE of true vs. predicted ratings, normalised over the tasks, with a tolerance of $10^{-5}$, averaged over 100 runs.

The results are outlined in Table \ref{table:lenk-mtl}.  
The cluster norm and centered spectral $k$-support outperformed the other penalties considerably in all regimes, with the \emph{c-cn} performing slightly better out of the two. 
The improvement due to centering compares favourably to similar results for the trace norm in \citep{Argyriou2008}.



\begin{table}
\caption{Multitask learning clustering on Lenk dataset.
}
\label{table:lenk-mtl}
\vskip 0.15in
\begin{center}
\begin{small}
\begin{tabular}{|l|l|c|c|c|c|}
\hline
	dataset  &  norm  & test error           & $r$      & $k$ & $a$      \\ 
\hline
\hline
5 per        & tr      & 2.2238 (0.17)          & 9      &      -&-            \\ \cline{2-6} 
task         & rn      & 2.1196 (0.06)          & 11  &      -&-            \\ \cline{2-6} 
            & ks      & 2.1679 (0.07)          & 10    & 1.01  &-             \\ \cline{2-6} 
             & box     & 2.1637 (0.07)           & 14 & 1.00 & 0.0001         \\ \cline{2-6} 
             & c-ks     & \textbf{1.9309 (0.04)}  & 11 & 1.01  &-             \\ \cline{2-6} 
             & c-cn      & \textbf{1.9219 (0.05)}  & 14 & 1.82 & 0.0006    	 \\ 
\hline
\hline
8 per       & tr       & 1.9110 (0.05)          & 14   &         -&-        \\ \cline{2-6} 
task         & en      & 1.8983 (0.05)          & 14 &         -&-           \\ \cline{2-6} 
           & ks      & 1.9158 (0.05)          & 13   & 1.04     &-         \\ \cline{2-6} 
		     & box     & 1.9061 (0.05)		   & 14  & 1.02 & 0.0003   	 \\ \cline{2-6} 
             & c-ks     & \textbf{1.7915 (0.03)} & 13    & 2.06  &-             \\ \cline{2-6} 
             & c-cn      & \textbf{1.7912 (0.03)} & 14   & 2.06 & 0.00006       \\ 
\hline
\hline
12 per       & tr      & 1.8702 (0.05)           & 8   &       -&-            \\ \cline{2-6} 
 task       & en       & 1.8642 (0.05)           & 8   &       -&-            \\ \cline{2-6} 
           & ks      & 1.7620 (0.07)            & 13  & 1.02   &-          \\ \cline{2-6} 
             & box     & 1.7194 (0.02)           & 14 & 1.01 & 0.0003        \\ \cline{2-6} 
             & c-ks      & \textbf{1.6813 (0.02)} & 14    & 2.61  &-             \\ \cline{2-6} 
             & c-cn      & \textbf{1.6744 (0.02)}  & 14 & 1.00 & 0.0009        \\ 
\hline
\end{tabular}
\end{small}
\end{center}
\vskip -0.1in
\end{table}


\section{Conclusion}\label{sec:conclusion}
We considered a family of regularizers and studied a special case of the parameter set.
We showed that the $k$-support norm belongs to this family, and we showed that the matrix extension of the $k$-support is a special case of the cluster norm for multitask learning.
We derived an efficient computation of the norm and an improved computation of the proximity operator of the square norm.
Furthermore we presented a method to solve matrix learning problems with the centered cluster norm using proximal gradient methods. 

Our empirical results confirm the improved computational performance of our proposed proximity operator algorithm, which enabled us to efficiently solve matrix problems using the spectral penalties outlined herein.  
Our experiments for matrix completion indicate that the spectral $\Theta$-norm generally outperforms the trace norm and the elastic net.
For multitask learning, the centered cluster norm and the centered spectral $k$-support norm considerably outperformed the trace norm and elastic net. 
Given the simpler form of the spectral $k$-support, the results suggest that on datasets exhibiting clustering structure this norm would be a viable alternative to the trace norm. 

In future work we would like to derive oracle inequalities and Rademacher complexities for the spectral penalty. 
It is also of interest to study the broader family of penalties, which admits natural extensions to a general family of operator norms, as well as their relationship to other known families of regularizers.

\subsubsection*{Acknowledgements}
We would like to thank Andreas Argyriou, Andreas Maurer and Charles Micchelli for useful discussions.  Part of this work was supported by EPSRC Grant EP/H027203/1 and Royal Society International Joint Project 2012/R2.

\newpage
\appendix

\section{Appendix}
In this appendix we provide proofs of some of the results stated in the main body of the paper, and collect some auxiliary results.

\begin{proof}[Proof of Theorem \ref{thm:theta-is-norm}]
Consider the expression for the dual norm. 
For fixed $\theta$, the objective $f$ in the supremum 
is positive, zero only for $u=0$, absolutely homogeneous of order one and convex.  The triangle inequality holds since for such $f$ by homogeneity and convexity we have
\begin{align*}
f(u+v) &= 2 f\left( \frac{1}{2}u + \frac{1}{2}v \right) \leq 2 \frac{1}{2}f(u) + 2 \frac{1}{2}f(v) \\
&= f(u) + f(v).
\end{align*}
Since the supremum of convex functions is convex, the dual $\Theta$ expression is indeed a norm.
With respect to the primal expression, we recall the definition of the Fenchel conjugate of $f:\mathbb{R}^d\rightarrow\mathbb{R}$:
\begin{align*}
f^*(u) = \sup_{w \in \mathbb{R^d}} \langle u,w \rangle-f(w).
\end{align*} 
It is a standard result from convex analysis that for any norm $\Vert \cdot \Vert$, the Fenchel conjugate $f^*$ of $f(w) = \frac{1}{2}\Vert w \Vert^2$ satisfies $f^*(u) = \frac{1}{2}\Vert u \Vert_*^2$, 
where $\Vert \cdot \Vert_*$ is the corresponding dual norm (see, e.g. \citep[p. 93]{Boyd2004}). 
Furthermore, for any norm, the biconjugate $f^{**}$ is equal to $f$ \citep[Th. 4.2.1]{Borwein2000}. 
Applying this to the dual norm we have
\begin{align*}
\frac{1}{2}\Vert w \Vert_{\Theta}^2 &= f(w) = \sup_{u} \left\{ \langle w,u \rangle - f^*(u)\right\}\\
&= \sup_{u} \left\{ \sum_{i=1}^d w_i u_i - \frac{1}{2} \sup_{\theta} \sum_{i=1}^d \theta_i u_i^2 \right\}\\
&= \sup_{u} \inf_{\theta} \left\{ \sum_{i=1}^d \left( w_i u_i - \frac{1}{2} \theta_i u_i^2\right)  \right\}.	
\end{align*}
This is a minimax problem in the sense of von Neumann \citep[see e.g.][]{Bertsekas2003}, and we can exchange the order of the $\inf$ and the $sup$, and solve the latter (which is in fact a minimum) componentwise.  
The gradient with respect to $u_i$ is zero for $u_i =\frac{w_i}{\theta_i}$, and substituting this into the objective we get
\begin{align*}
\frac{1}{2}\Vert w \Vert_{\Theta}^2 &=  \frac{1}{2} \inf_{\theta} \sum_{i=1}^d \frac{w_i^2}{\theta_i},
\end{align*}
and we recover the infimum expression.  It follows that eq. (\ref{eqn:theta-primal}) defines a norm, and the two norms are duals of each other as required.
\end{proof}

In order to prove Proposition \ref{prop:infconv-abc}, we will prove the result under more general assumptions and derive Proposition \ref{prop:infconv-abc} as a special case. Our proof technique is essentially taken from \citep[Theorem 7]{Maurer2012}.

We first show the following.
\begin{lemma}\label{lem:max-of-norms}
Let $T_n \in S_d$, the set of real valued symmetric $d \times d$ matrices, for $n=1, \ldots, N$, and assume for every $w \in \R^d$, $w \neq 0$, there exists $T_n$ such that $T_n w \neq 0$. Then 
\begin{align}
\max_{n=1}^N \Vert T_n w \Vert_2  \label{eqn:max-of-norms-is-norm}
\end{align}
is a norm and the dual norm is given by 
\begin{align*}
\inf_{v } \left\{ \sum_{n=1}^N \Vert T_n^+ v_n \Vert_2 : \sum_{n=1}^N v_n = w , v_n \in \textrm{Range}(T_n) \right\}
\end{align*}
where $T_n^+$ is the pseudo-inverse of $T_n$.
\end{lemma}

\begin{proof}
We first show that \eqref{eqn:max-of-norms-is-norm} defines a norm.  
Positivity and homogeneity are immediate.  
Non-degeneracy ($\Vert w \Vert = 0 \Leftrightarrow w=0$) follows from the assumption that for every $w \in \R^d$, $w \neq 0$, there exists $T_n$ such that $T_n w \neq 0$. 
To show the triangle inequality, note
\begin{align*}
\max_{n=1}^N \Vert T_n(u+v) \Vert_2 &\leq \max_{n=1}^N \left( \Vert T_n u \Vert_2 + \Vert T_n v \Vert_2\right)\\
&\leq \max_{n=1}^N  \Vert T_n u \Vert_2 + \max_{n=1}^N \Vert T_n v \Vert_2,
\end{align*}
and we are done.  
To derive the dual norm we need to compute 
\begin{align*}
\Vert w\Vert_* = \sup_{u} \{ \langle w,u\rangle: \Vert T_n^+ u \Vert_2 \leq 1,~n=1,\dots,N \}.
\end{align*}
The same assumption implies that $T=\sum_{n-1}^N T_n$ has non trivial nullspace, hence is invertible, and 
\begin{align*}
w = T^{-1} \sum_{n=1}^N T_n w = \sum_{n=1}^N T^{-1} T_n w = \sum_{n=1}^N v_n,
\end{align*}
where $v_n = T^{-1}T_n w$. 
Letting $v=(v_1, \ldots, v_N)$, noting that for $v_n \in \textrm{Range}(T_n)$, we have $v_n=T_n T_n^+ v_n$, and using symmetry of $T_n$, we have 
\begin{align*}
\langle w,u\rangle &=  \sum_{n=1}^N \langle T_n T_n^+ v_n,u\rangle \\
&= \sum_{n=1}^N \langle T_n^+ v_n, T_n u \rangle  \\
&\leq \sum_{n=1}^N \Vert T_n^+ v_n\Vert_2 \Vert T_n u\Vert_2 \\
&\leq \sum_{n=1}^N \Vert T_n^+ v_n \Vert_2,
\end{align*}
where the penultimate inequality follows by Cauchy Schwarz.  
Furthermore all inequalities are tight for the choice $v_n = T_n^+ u$, if $u \in \argmax_{n=1}^N \Vert T_n u\Vert_2$, and $v_n =0$ otherwise. The result follows.
\end{proof}

Note that the assumption that $T_n \in S_d$ is not restrictive since the matrices in eq. \eqref{eqn:max-of-norms-is-norm} inside the Frobenius norm appear in the form $T_n\trans T_n$. 
For diagonal matrices we have the following corollary of Lemma \ref{lem:max-of-norms}. 

\begin{corollary}\label{cor:diag-Tn}
Let $T_n = {\rm diag}(t_n)$ for some $t_n \in \R^d_+$ and assume for every $w \in \R^d$, $w \neq 0$, there exist $n$ such that $T_n w \neq 0$. Then 
$$
\max_{n=1}^N \|T_n u\|_2
$$
is a norm and the dual norm is given by 
$$
\inf_{v \in \V} \left\{\sum_{n=1}^N \|T_n^+ v_n\|_2 : \sum_{n=1}^N v_n = w\right\}
$$
where $\V = \{(v_n)_{n=1}^N: v_n \in \R^n, {\rm supp}(v_n ) \subseteq {\rm supp}(t_n)\}$ and $T_n^+$ is the pseudo-inverse of $T_n$.
\end{corollary}

We can now prove Proposition \ref{prop:infconv-abc}. 

\begin{proof}[Proof of Proposition \ref{prop:infconv-abc}]
Note that Corollary \ref{cor:diag-Tn} applies to the $\Theta$-norm in the case that the closure of $\Theta$ has a finite number of extreme points  $t_1,\dots,t_N$.  
Now observe that the extreme points of the set $\Theta$ in equation \eqref{eqn:theta-abc-def} for the given choice of $a,b$ and $c$ are of the form $1 a + 1_g (b-a)$ for $g \in \G_k$, where recall $\G_k$ is the set of all subsets of $\{1,\dots,d\}$ of cardinality $k$, $1$ is the vector of all ones, and $(1_g)$ is the indicator vector for set $g$. 
Let $N= |\G_k|$ and identify $n$ with $g$. Note for this choice that, for $x \in \R^d$
$$
\|T_g  x\|_2 = \sqrt{\sum_{i \in g}^d \frac{x_i^2}{b} + \sum_{i \notin g}^d \frac{x_i^2}{a}},
$$
and the proposition is proved. 
\end{proof}

The following result illustrates the relationship between the unit balls relating to $\Vert T_n \cdot \Vert$ and the unit ball of $\max_n \Vert T_n \cdot \Vert_2$ as defined in eq. \eqref{eqn:max-of-norms-is-norm}.

\begin{proposition}
Let $\Vert \cdot \Vert$ be the norm defined on $\mathbb{R}^d$ by 
\begin{align*}
\Vert w \Vert &= \inf_{ v \in \mathcal{V} } \left\{ \sum_{n=1}^N \Vert T_n^+ v_n \Vert : \sum_{n=1}^N v_n = w, v_n \in \textrm{Range}(T_n)  \right\},
\end{align*}
for $T_n \in S_d$ with pseudoinverse $T_n^+$.  
Let $B = \{ w: \Vert w \Vert \leq 1 \}$ be the corresponding unit ball, and let $A= \textrm{co}\left(\bigcup_{n=1}^N A_n\right)$, where 
\begin{align*}
A_n = \{ w: \Vert T_n^+ w \Vert \leq 1, w \in \textrm{Range}(T_n) \}.
\end{align*}  
Then $B = A$.
\end{proposition}

\begin{proof}
We first show $A \subset B$.  
Let $w \in A$.  By definition of the convex hull, there exist $w_n$ and $\lambda_n \in [0,1]$ ($n=1, \ldots, N$) such that $w_n \in \textrm{Range}(T_n)$, $w= \sum_{n=1}^N \lambda_n w_n$, and $\sum_{n=1}^N \lambda_n = 1$. 
Defining $\hat{v}_n = \lambda_n w_n$, we have $w= \sum_{n=1}^N v_n$.  It follows that
\begin{align*}
\Vert w \Vert &= \inf_{\{ v_n \}} \left\{ \sum_{n=1}^N \Vert T_n^+ v_n \Vert : \sum_{n=1}^N v_n = w, v_n \in \textrm{Range}(T_n)  \right\}\\
&\leq  \sum_{n=1}^N \Vert T_n^+ \hat{v}_n \Vert =  \sum_{n=1}^N \lambda_n \Vert T_n^+ w_n \Vert \leq \sum_{n=1}^N \lambda_n = 1.
\end{align*}

To show $B \subset A$, let $w \in B$, so $\Vert w \Vert \leq 1$.
Then there exists a sequence $\{v^k_1, \ldots, v^k_n\}$ such that for each $k$, $\sum_{n=1}^N v^k_n  = w$ and $v^k_n \in \textrm{Range}(T_n)$, and we have 
\begin{align*}
\Vert w \Vert = \lim_{k \rightarrow \infty} \sum_{n=1}^N \Vert T_n^+ v^k_n \Vert.
\end{align*}
Let $w^k_n = \frac{v^k_n}{\Vert T_n^+ v^k_n \Vert}$ and $\lambda_n^k = \Vert T_n^+ v^k_n \Vert$.  
Then for all $k$, $\sum_{n=1}^N \lambda^k_n = 1$, $\sum_{n=1}^N \lambda^k_n w^k_n = w$, and $w^k_n \in A_n$, and taking the limit we get the desired result. 
\end{proof}

\begin{proof}[Proof of Theorem \ref{thm:computation-of-theta-norm}]
We solve the constrained optimization problem
\beq
\inf_{\theta} \bigg\{ \sum_{i=1}^d \frac{w_i^2}{\theta_i} : a \leq \theta_i \leq b, \sum_{i=1}^d \theta_i \leq c \bigg\}.  
\label{eqn:norm-objective-supp}
\eeq
To simplify notation we assume without loss of generality that the $w_i$ are ordered non increasing, and it follows by Lemma \ref{lem:ordering} that the $\theta_i$ are ordered non increasing.  
We further assume without loss of generality that $w_i \ne 0$ for all $i$, and $c \leq db$ (see Remark \ref{rem:wi-not-zero}). 
The objective is continuous and we take the infimum over a closed bounded set, so a solution exists, the solution is a minimum, and it is unique by strict convexity.  
Furthermore, since $c \leq db$, the sum constraint will be tight at the optimum.

The Lagrangian is given by 
\begin{align*}
L(\theta,\alpha) &= \sum_{i=1}^d \frac{w_i^2}{\theta_i}  + \frac{1}{\alpha^2} \left( \sum_{i=1}^d \theta_i-c\right)
\end{align*}
where $1/\alpha^2$ is a strictly positive multiplier to be chosen to make the sum constraint tight, call this value $\alpha^*$.
Now let $\theta^*$ be the minimizer of $L(\theta, \alpha^*) $ over $\theta$ subject to $a\leq \theta_i \leq b$. 
Then $\theta^*$ solves equation \eqref{eqn:norm-objective-supp}.  
To see this, note that we have for any $\theta$, $L(\theta^*,\alpha^*) \leq L(\theta, \alpha^*)$, so we have
\begin{align*}
\sum_{i=1}^d \frac{w_i^2}{\theta^*_i}   \leq
\sum_{i=1}^d \frac{w_i^2}{\theta_i}  + \frac{1}{(\alpha^*)^2} \left( \sum_{i=1}^d \theta_i-c\right).
\end{align*}
As this holds for all $\theta \in [a,b]^d$, then it holds for all such $\theta$ where additionally the constraint $\sum_{i=1}^d \theta_i \leq c$ holds.  
In this case the second term on the right hand side is at most zero, so we have for all such $\theta$
\begin{align*}
\sum_{i=1}^d \frac{w_i^2}{\theta^*_i}   \leq \sum_{i=1}^d \frac{w_i^2}{\theta_i},
\end{align*}
whence it follows that $\theta^*$ is the minimizer of \eqref{eqn:norm-objective}.

We can therefore solve the original problem by minimizing the Lagrangian over the box constraint.
Due to the coupling effect of the multiplier, the problem is separable, and we can solve the simplified problem componentwise using Lemma \ref{lem:mmp-box-solution}.
It follows that 
\begin{align}
\theta_i = 
\begin{cases}
a, \quad &\text{if} \quad \alpha < \frac{a}{\vert w_i \vert},\\
\alpha \vert w_i \vert, \quad &\text{if} \quad \frac{a}{\vert w_i \vert}\leq \alpha \leq \frac{b}{\vert w_i \vert} ,\\
b, \quad &\text{if}\quad  \alpha > \frac{b}{\vert w_i \vert} ,
\end{cases}\label{eqn:form-of-optimal-theta-abc-supp}
\end{align}
where $\alpha>0$ is such that $\sum_{i=1}^d \theta_i(\alpha) = c$. 
The minimizer then has the form 
\begin{align*}
\theta = (\underbrace{b, \ldots, b}_q, \theta_{q+1}, \ldots, \theta_{d-\ell}, \underbrace{a, \ldots, a}_{\ell})
\end{align*}
where $q,\ell \in \{0, \ldots, d\}$ are determined by the value of $\alpha$, which satisfies
\begin{align*}
S(\alpha) = \sum_{i=1}^d \theta_i(\alpha) =q b + \sum_{i=q+1}^{d-\ell} \alpha \vert w_i \vert +  \ell a = c,
\end{align*}
i.e. $\alpha = \frac{p}{\sum_{i=q+1}^{d-\ell}\vert w_i \vert}$ where $p=c-qb -\ell a$.
The value of the norm follows by substituting $\theta$ into the objective and we get
\begin{align*}
\Vert w \Vert_{\Theta}^2 
&= \sum_{i=1}^{q} \frac{\vert w_i \vert^2}{b^2} + 
\frac{1}{p} \Big(\sum_{i=q+1}^{d-\ell} \vert w_i\vert\Big)^2 +  \sum_{i=d-\ell+1}^{d} \frac{\vert w_i\vert^2}{a^2} \\
&= \frac{1}{b^2} \Vert w_Q \Vert_2^2  + 
\frac{1}{p} \Vert w_I\Vert_1^2 +
\frac{1}{a^2} \Vert w_L \Vert_2^2, 
\end{align*}
as required. 
We can further characterise $q$ and $\ell$ by considering the form of $\theta_i$.  
By construction we have $\theta_q \geq b >\theta_{q+1}$ and $\theta_{d-\ell} > a \geq \theta_{d-\ell+l}$, or equivalently
\begin{align*}
\frac{\vert w_q \vert}{b} &\geq \frac{1}{p}\sum_{i=q+1}^{d-\ell}\vert w_i \vert > \frac{\vert w_{q+1} \vert}{b}\\
\frac{\vert w_{d-\ell} \vert}{a} &\geq \frac{1}{p}\sum_{i=q+1}^{d-\ell}\vert w_i \vert > \frac{\vert w_{d-\ell+1} \vert}{a},
\end{align*}
as required. 
\end{proof}


\begin{remark}\label{rem:wi-not-zero}
The case where some $w_i$ are zero follows from the case that we have considered in the theorem.  
If $w_i=0$ for $n < i \leq d$, then clearly we must have $\theta_i=a$ for all such $i$.  
We then consider the $n$-dimensional problem of finding $(\theta_1, \ldots, \theta_{n})$ that minimizes $\sum_{i=1}^{n} \frac{w_i^2}{\theta_i}$, subject to $a \leq \theta_i \leq b$, and $\sum_{i=1}^{n} \theta_i \leq c'$, where $c'=c-(d-n)a$.  
As $c \geq da$ by assumption, we also have $c' < na$, so a solution exists to the $n$ dimensional problem.  
If $c' < bn$, then a solution is trivially $\theta_i=b$ for all $i=1\ldots n$.  
In general, $c' \geq bn$, and we proceed as per the proof of the theorem. 
Finally, a $\theta$ that solves the original $d$-dimensional problem will be given by $(\theta_1, \ldots, \theta_{n}, a, \ldots, a)$. 
\end{remark}



The following auxiliary result is used in Theorem \ref{thm:computation-of-theta-norm}.

\begin{lemma}\label{lem:ordering}
Let $w_1 \geq \ldots \geq w_d > 0$, $\Theta =\{\theta: a \leq \theta_i \leq b, \sum_{i=1}^d \theta_i \leq c\}$ ($0<a<b$), and let $\theta^* = \argmin_{\theta \in \Theta} \sum_{i=1}^d \frac{w_i^2}{\theta_i}$. 
Then $\theta_1^* \geq \ldots \geq \theta_d^*$.
\end{lemma}

\begin{proof}
Let $i<j$, so $w_i \geq w_j$, and suppose that $\theta^*_i < \theta_j^*$.  Define $\hat{\theta}$ to have identical elements to $\theta^*$, except with the $i$ and $j$ elements exchanged.  
As $\theta^*$ is the unique $\argmin$ of a strictly convex function, comparing the objectives we have
\begin{align*}
\frac{w_i^2}{\theta_i^*} + \frac{w_j^2}{\theta_j^*} &< \frac{w_i^2}{\hat{\theta}_i} + \frac{w_j^2}{\hat{\theta}_i} = \frac{w_i^2}{\theta_j^*} + \frac{w_j^2}{\theta_i^*}\\
w_i^2 \left(\frac{1}{\theta_i^*} -\frac{1}{\theta_j^*} \right) &<  w_j^2 \left( \frac{1}{\theta_i^*} - \frac{1}{\theta_j^*}\right).
\end{align*}
Dividing by $\frac{1}{\theta_i^*} -\frac{1}{\theta_j^*}>0$ we get $w_i< w_j$, a contradiction.
\end{proof}


We further make use of the following result, which follows from \citep[Theorem 3.1]{Micchelli2011}.  

\begin{lemma}\label{lem:mmp-box-solution}
Let $0<a<b$, and let $f(\theta) = \sum_{i=1}^d \left(\frac{w_i^2}{\theta_i} + \beta^2 \theta_i\right)$ and consider $\min_{a\leq \theta_i \leq b} f(\theta)$.
Then the minimizer $\theta \in \mathbb{R}^d$ is given by 
\begin{align*}
\theta_i = 
\begin{cases}
a, \quad &\text{if } \frac{\vert w_i \vert}{\beta}< a,\\
\frac{\vert w_i \vert}{\beta}, &\text{if } a \leq \frac{\vert w_i \vert}{\beta} \leq b,\\
b, \quad &\text{if }  \frac{\vert w_i \vert}{\beta}> b.
\end{cases}
\end{align*}
\end{lemma}

\begin{proof}  
For fixed $w$, the objective function is strictly convex on $\mathbb{R}^d_{++}$ and has a unique minimum on $(0,\infty)$ (see Figure 1.b in \citep{Micchelli2011} for a one-dimensional illustration).  
The problem is separable and we solve
\begin{align*}
\min_{a \leq \theta_i \leq b} \frac{w_i^2}{\theta_i} + \beta^2,
\end{align*}
which has gradient with respect to $\theta_i$ given by $-\frac{w_i^2}{\theta_i^2}+\beta$.  
This is zero for $\theta_i = \frac{\vert w_i \vert}{\beta} =:\theta^*_i$, strictly negative below $\theta_i^*$ and strictly increasing above $\theta^*_i$.   
Considering these three cases we recover the expression in statement of the lemma.
\end{proof}

\bibliographystyle{plainnat}
\renewcommand{\bibname}{References} 
\bibliography{jabrefmaster.bib} 

\end{document}